\documentclass[sigconf]{acmart}
\usepackage{amsmath}
\usepackage{amsthm}
\usepackage{color}
\usepackage{natbib}

\newtheorem{thm}{Theorem}
\numberwithin{equation}{section}

\AtBeginDocument{
  }

\copyrightyear{2026}
\acmYear{2026}
\setcopyright{cc}
\setcctype{by}
\acmConference[KDD '26]{Proceedings of the 32nd ACM SIGKDD Conference on Knowledge Discovery and Data Mining V.2}{August 09--13, 2026}{Jeju Island, Republic of Korea}
\acmBooktitle{Proceedings of the 32nd ACM SIGKDD Conference on Knowledge Discovery and Data Mining V.2 (KDD '26), August 09--13, 2026, Jeju Island, Republic of Korea}
\acmDOI{10.1145/3770855.3817659}
\acmISBN{979-8-4007-2259-2/2026/08}

\begin{document}

\title{A More Accurate Algorithm Comparison through A/B Testing using Offline Evaluation Methods}

\author{Koki Konishi}
\affiliation{
  \institution{Hakuhodo DY Holdings Inc.}
  \city{Minato City}
  \state{Tokyo}
  \country{Japan}}
\email{koki.konishi@hakuhodo.co.jp}

\author{Masataka Ushiku}
\affiliation{
  \institution{Hakuhodo DY Holdings Inc.}
  \city{Minato City}
  \state{Tokyo}
  \country{Japan}}
\email{masataka.ushiku@hakuhodo.co.jp}

\author{Yuta Saito}
\affiliation{
  \institution{Cornell University}
  \city{Ithaca}
  \state{New York}
  \country{United States}}
\email{ys522@cornell.edu}

\renewcommand{\shortauthors}{Koki Konishi, Masataka Ushiku, \& Yuta Saito}

\begin{abstract}
A/B testing is the gold standard for selecting better algorithms in online services. While offline evaluation has attracted attention as a safer alternative due to the high experimental costs and the potential risk of degrading user experience and revenue in A/B testing, it is widely recognized that the estimation accuracy of offline evaluation is substantially lower than that of A/B testing. As a result, final decisions on algorithm selection are typically made through A/B testing.
Contrary to this conventional view, we reveal a counterintuitive phenomenon in which A/B testing can produce a higher algorithm selection error rate than offline evaluation. This occurs because the sample mean estimator used in A/B testing does not induce positive correlation, which plays a crucial role in reducing critical selection errors, namely underestimating the truly superior algorithm and overestimating the truly inferior one. In contrast, offline evaluation methods unintentionally generate this beneficial correlation by relying on shared offline data when estimating and comparing the performance of multiple algorithms.
Building on this insight, we propose a novel estimator that intentionally induces positive correlation to improve algorithm selection in A/B testing. The key idea is to introduce a hypothetical middle algorithm and to estimate the performance difference between algorithms A and B in a stepwise manner, first between A and the middle algorithm and then between the middle algorithm and B, using shared data at each step. This approach enables the application of offline evaluation techniques in each step, thereby inducing positive correlation and reducing critical selection errors. Furthermore, we derive the optimal middle algorithm regarding the resulting variance and analyze its advantages over existing methods through bias-variance analysis. Experiments on real-world data demonstrate that the proposed estimator achieves the same selection error rate as existing approaches while using only one half of the A/B testing data, indicating a twofold improvement in sample efficiency.
\end{abstract}




\maketitle

\newcommand\kddavailabilityurl{https://doi.org/10.5281/zenodo.20482776}
\ifdefempty{\kddavailabilityurl}{}{
\begingroup\small\noindent\raggedright\textbf{Resource Availability:}\\
The source code of this paper has been made publicly available at \url{\kddavailabilityurl}.
\endgroup
}

\begin{figure}[ht]
    \centering
    \includegraphics[width=0.80\linewidth]{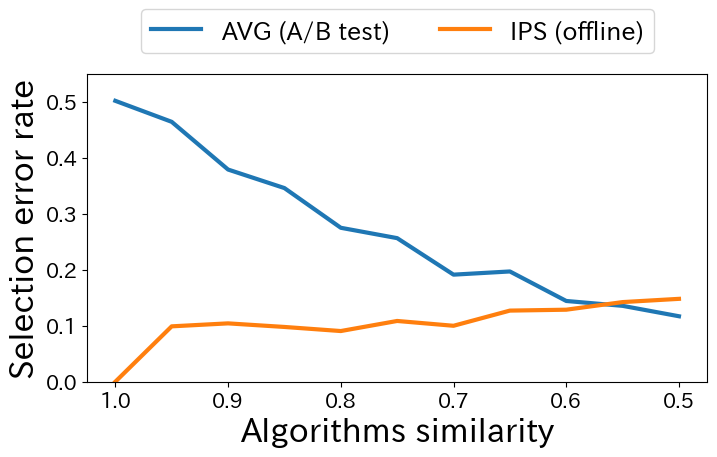}
    \vspace{-3mm}
    \caption{Comparison of selection error rates between A/B testing and offline evaluation.}
    \vspace{-5mm}
    \label{fig:pre-experiment-error}
\end{figure}

\section{Introduction}
Data-driven decision making is essential for many online services, such as recommender systems, search engines, and digital advertising.
These decision-making problems are commonly formulated using the Contextual Bandit framework~\cite{Li2011-jb, Durand2018-qg, Gruson2019-es}.
In the development of these algorithms, it is crucial to be able to identify the better-performing algorithm based on key metrics (e.g., revenue, click-through rate).
In this context, A/B testing is considered the gold standard~\cite{Gomez-Uribe2016-dw, Larsen2023-ju}.
The A/B testing process is divided into two phases: \textit{data collection} and \textit{estimation}~\cite{Kohavi2012-wk}.
In the data collection phase, the system randomly assigns users to multiple groups (for example, group A and group B) in order to isolate the effect of algorithmic differences, and it collects interaction logs from each group.
In the estimation phase, the system typically computes the sample mean of the outcomes or rewards for groups A and B and then compares these means to identify the better.
In this paper, we refer to the data collected from both groups A and B as A/B testing data and denote the sample mean as the AVG estimator.

A/B testing is a powerful and reliable algorithm selection method, but it is not perfect. 
If the tested algorithm is of low quality, it carries the risk of harming user experience and revenue during the experimental period.
For that reason, offline evaluation has gained attention as a safe alternative, as it uses only historical log data, which we refer to as offline data, without deploying a new algorithm~\cite{Gruson2019-es, Castells2022-jr, gilotte2018offline}.
A representative method, namely the Inverse Propensity Scoring (IPS) estimator~\cite{Horvitz1952-ws, precup2000eligibility}, unbiasedly estimates the performance of a new algorithm by statistically re-weighting the observed data.
However, it is widely recognized that the estimation accuracy of offline evaluation is far inferior to that of A/B testing, which collects new data in an online environment~\cite{gilotte2018offline, saito2022counterfactual}.
Therefore, offline evaluation typically performs preliminary screening, and A/B testing often makes the final decision on algorithm selection~\cite{Gomez-Uribe2016-dw}.

\begin{table*}
    \centering
    \caption{Proportion of incorrect selections for each quadrant out of the total trials at a similarity of 0.80.}
    \vspace{-3mm}
    \begin{tabular}{c|c|c|c|c}
        & overestimates for algo A \& & overestimates for algo A \& & underestimates for algo A \& & underestimates for algo A \& \\
        & overestimates for algo B & underestimates for algo B & overestimates for algo B & underestimates for algo B \\
        \hline
        AVG & 3.71\% & \textbf{20.26\%} & 0\% & 3.52\% \\
        IPS & 1.15\% & \textbf{3.39\%} & 0\% & 4.51\%
    \end{tabular}
    \label{tab:plot-num}
\end{table*}

\paragraph{\textbf{Pre-experiment: online vs. offline in algorithm selection}}
This research is motivated by the discovery of an interesting phenomenon in our pre-experiment that questions the conventional understanding that "A/B testing is the gold standard".
In that pre-experiment, we built an A/B testing environment and then selected the better algorithm using two methods: AVG using A/B testing data, and IPS using only offline data from algorithm A.
The primary metric for this comparison is the selection error rate.
This rate is defined as the proportion of 10,000 trials in which the estimator incorrectly identified the truly worse algorithm as the better one based on a finite sample.
As the true performance of algorithm B is higher in this experiment, a mismatch occurs when the estimators incorrectly select algorithm A.
Figure \ref{fig:pre-experiment-error} shows the comparison results of the selection error rate for each method.
The horizontal axis represents the similarity between the two algorithms\footnote{The parameter $\mu_{rate}$ ranges from 0.0 to 1.0; a larger $\mu_{rate}$ means the algorithm tends to select actions with higher expected rewards. In this paper, the similarity between the two algorithms is defined as $1 - |\mu_{rate_A} - \mu_{rate_B}|$}.
The results indicate that \textbf{AVG using A/B testing data, which should be the gold standard, is inferior to IPS using offline data in terms of selection error rate in most cases.}

The key to understanding this counterintuitive phenomenon is Figure \ref{fig:correlation}.
It visualizes scatter plots of the estimates of algorithm performance based on 10,000 trials with varying random seeds at an algorithm similarity of 0.80.
\begin{figure}
    \centering
    \includegraphics[width=0.95\linewidth]{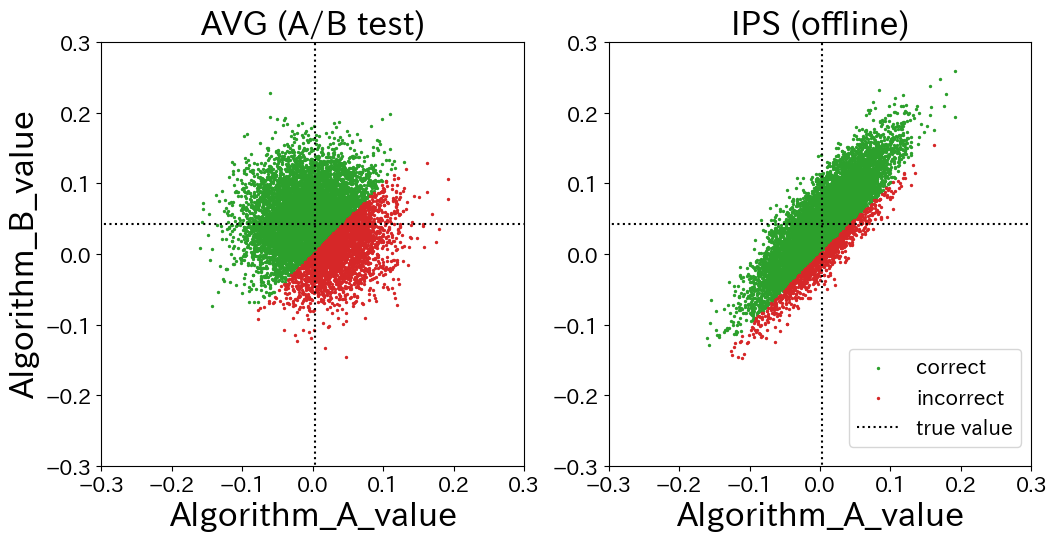}
    \vspace{-3mm}
    \caption{The scatter plot of algorithm performance values estimated by AVG and IPS.}
    \label{fig:correlation}
    \vspace{-5mm}
\end{figure}
The horizontal and vertical axes represent the estimates for algorithms A and B, respectively.
The diagonal line $y=x$ represents the decision boundary: points above this indicate that the estimates for algorithm B are higher than those for algorithm A.
In this pre-experiment, since the true performance of algorithm B is higher, the points above the boundary line are the correct selections.
We plotted the correct selections as green dots and the incorrect selections as red dots.
The scatter plot reveals that the vertical dispersion of AVG is smaller than that of IPS, indicating superior stability in evaluating individual algorithms.
However, despite this lower variance, the proportion of incorrect selections is 27.49\% for AVG, compared to only 9.05\% for IPS.
The shape of the plot reveals the reason for this discrepancy.
Specifically, the plot of AVG is circular due to the estimates derived from independent A/B testing data, whereas the plot of IPS is elliptical, indicating a positive correlation arising as a byproduct of using shared offline data.
Table \ref{tab:plot-num} details the specific types of selection error rates.
Notably, AVG generates more trials in the bottom-right quadrant, namely the underestimation of the truly better algorithm B and the overestimation of the truly worse algorithm A, which leads to incorrect selections.
This also indicates that the positive correlation, a byproduct of offline evaluation, avoids these critical errors.
This implication is not merely the advantage of offline evaluation; rather, AVG fails to leverage its available A/B testing data effectively.

\paragraph{\textbf{Contributions}}
The primary objective of this work is to propose a novel estimator, referred to as the Middle-In-Difference (MID) estimator, that selects the superior algorithm in A/B testing with a lower selection error rate than the conventional AVG estimator.
Moreover, as an algorithm selection method designed for A/B testing, MID is naturally expected to outperform offline evaluation methods in terms of selection accuracy, even though such a guarantee is not achieved by AVG.
MID fully exploits A/B testing data by intentionally inducing positive correlation within the A/B testing framework.
As depicted in Figure~\ref{fig:proposed-method}, we introduce a hypothetical middle algorithm M positioned between algorithms A and B.
We then estimate the performance differences between A and M, and between M and B, using offline evaluation techniques, and aggregate these estimates to obtain the performance difference between A and B.
\begin{figure}
    \centering
    \includegraphics[width=\linewidth]{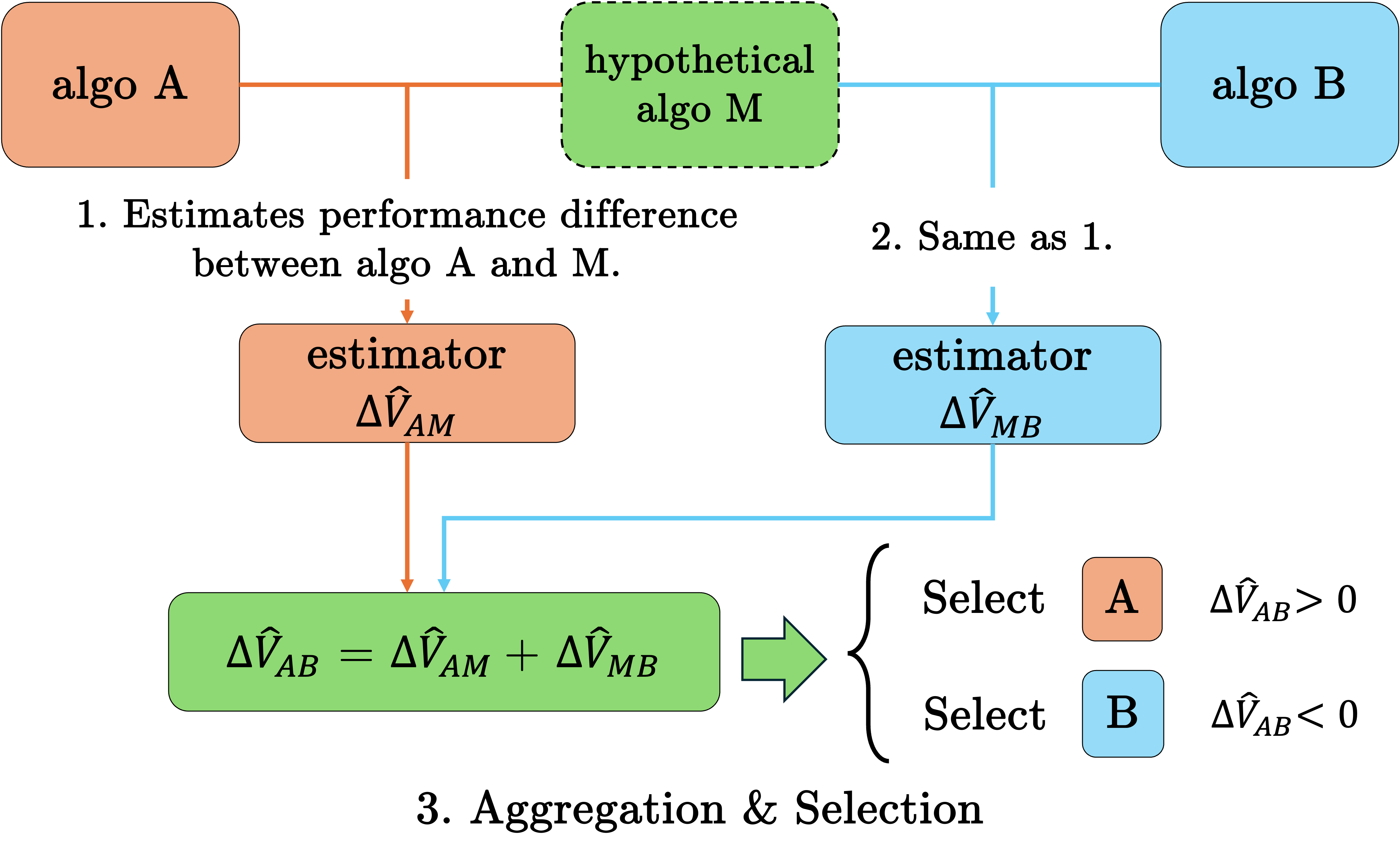}
    \vspace{-5mm}
    \caption{Diagram of proposed method: Stepwise estimation of the algorithm performance via a middle algorithm.}
    \vspace{-5mm}
    \label{fig:proposed-method}
\end{figure}
From a theoretical perspective, we derive the optimal middle algorithm M that minimizes the variance, thereby reducing the algorithm selection error rate. Moreover, through a detailed bias--variance analysis, we demonstrate the statistical mechanism that underlies the superior performance of MID compared to existing estimators based on either A/B testing data or offline data.

To empirically validate the effectiveness of MID, we conduct comparative experiments using real-world video recommendation log data.
We construct an A/B testing environment and evaluate the selection error rate of MID against that of existing methods.
The experimental results show that MID consistently achieves lower selection error rates across a wide range of settings, and in particular, it achieves the same selection error rate as existing methods with only half the number of A/B testing samples, demonstrating a substantial improvement in data efficiency.

Our main contributions are summarized as follows.
\begin{itemize}
    \item We clarify the statistical mechanism demonstrating that the AVG estimator is not necessarily optimal for algorithm selection in A/B testing, and we show that strategically inducing positive correlation between estimators can substantially reduce incorrect algorithm selections.
    \item We propose a novel estimator, MID, that achieves a lower selection error rate than AVG typically used in A/B testing and standard offline evaluation methods by employing stepwise estimation through a hypothetical middle algorithm.
    \item We empirically demonstrate the effectiveness of MID using public real-world data, showing that it consistently yields lower selection error rates than existing approaches.
\end{itemize}

\section{Related Work}
In this section, we summarize the literature relevant to our study, focusing on A/B testing and offline evaluation.

\subsection{A/B Testing}
A/B testing, also known as Online Controlled Experiments (OCEs), is an essential tool for data-driven decision-making in various online platforms, such as recommender systems, search engines, and e-commerce websites\cite{Gomez-Uribe2016-dw, Kohavi2013-rx, Zhang2025-ep}.
It is widely recognized as the gold standard for evaluating the performance of algorithms\cite{Gomez-Uribe2016-dw, Kohavi2012-wk, Kohavi2013-rx, Zhang2025-ep, Kohavi2023-el}.
By randomly assigning users to different groups, A/B testing eliminates confounding factors and ensures that the difference in outcomes is attributed to the change in algorithms.
This randomization procedure guarantees the unbiasedness of the algorithm’s performance, providing accurate evaluation for the compared algorithms.

However, target metrics (e.g., click-through rate and revenue) in A/B testing are often highly variable\cite{Deng2013-gn, Guo2021-sr,Jeunen2024-mb}.
Since this variability requires large sample sizes to achieve sufficient statistical power, A/B testing is often time-consuming and costly.
To address this sample inefficiency, previous research has focused on variance reduction techniques.
Controlled-experiment Using Pre-Experiment Data (CUPED) is a basic method in this domain\cite{Deng2013-gn}.
This method improves the sensitivity of A/B testing by selecting pre-experiment covariates positively correlated with the outcome and estimating the difference between the outcome and the weighted covariates.
\citet{Guo2021-sr} expanded this concept by using machine learning to capture more complex correlations.
\citet{Lin2024-nl} proposed a variance reduction method that combines pre-experiment and in-experiment covariates.
While methods like CUPED reduce the variance by leveraging the positive correlation between outcomes and covariates, our proposed method uses the positive correlation between the performance estimators of the compared algorithms, which arises from offline evaluation.
In this study, we address this gap by proposing a novel estimator, which uses the correlation to improve algorithm selection accuracy.

\citet{Wan2022-kh-SafeExploration} proposed an algorithm evaluation and comparison method using offline evaluation techniques.
This method compares an existing algorithm with a designed efficient and safe algorithm instead of directly comparing the existing algorithm with a new one.
The alternative algorithm guarantees sufficient estimation accuracy without significantly degrading performance compared to the existing algorithm.
However, this method does not assume an A/B testing environment, which differs from the experimental setting of our study.
Similarly, \citet{Sakhi2025-practical-improvements} used offline evaluation techniques and improve A/B testing by focusing on algorithm similarity.
However, their problem setting differs from ours as they formulate A/B testing using the framework of Markov Decision Processes.
Furthermore, their primary objective is to improve the estimation accuracy of performance differences, whereas we aim to improve algorithm selection accuracy.

\subsection{Offline Evaluation}
While online experiments directly measure the performance of algorithms by deploying algorithms to real environments, they are often not a realistic approach due to time and implementation constraints.
Therefore, offline evaluation has gained attention for its ability to evaluate new algorithms using only historical log data.
The Inverse Propensity Score (IPS) estimator enables unbiased estimation by re-weighting observed rewards based on the action selection probabilities of the logging and target algorithms, though it often suffers from high variance.
The Direct Method (DM) estimator is another evaluation method using a reward model, which approximates the unknown expected reward function using historical log data.
The variance of DM is often lower than that of IPS, but it introduces bias from reward model misspecification.
The Doubly Robust (DR) estimator combines these methods to achieve unbiasedness and low variance.
In addition to these standard estimators, subsequent research has proposed various methods.

Leveraging these offline evaluation techniques, prior work has focused on the problem of algorithm comparison and selection\cite{Jeunen2024-vb, Sakhi2024-mt, Yang2022-rm, Konyushkova2021-if, Zhang2021-mr}.
For instance, \citet{Jeunen2024-vb} observed that in real-world applications, improvements in algorithms are typically incremental.
This leads to high similarity between logging and target algorithms, inducing a positive correlation between their estimators.
Based on this insight, they proposed a method to reduce estimation variance by estimating the performance differences.
However, the scope of their work is limited to offline algorithm selection, which is distinct from online experiments.
\section{Preliminaries}
This section formulates the problem of data-driven algorithm selection and discusses existing methods.

\subsection{Formulation of Algorithm Selection}
Data-driven decision-making problems in online services and business applications are commonly formulated within the contextual bandit framework~\cite{Li2011-jb, Durand2018-qg, Gruson2019-es}. The contextual bandit framework models a sequential decision-making process: determining the action that maximizes the reward based on a given situation (context).
This process assumes a flow where, first, a context $x \in \mathcal{X}$ (e.g., user demographics, behavioral history) is drawn from an unknown distribution $p(x)$.
In response, an algorithm $\pi(a \mid x)$ stochastically selects an action $a \in \mathcal{A}$ (e.g., an item such as a video or song).
Consequently, a reward $r$ (e.g., click-through, viewing time) is observed from an unknown distribution $p(r \mid x, a)$.
Here, we define the expected reward function for a specific action $a$ in a given situation $x$ as $q(x, a) := \mathbb{E}[r \mid x, a]$.
The performance of an algorithm $\pi$, denoted as $V(\pi)$, is often expressed as follows~\cite{Su2020-av, Sachdeva2020-ha}.
\begin{align}
    V(\pi) &:= \mathbb{E}_{p(x)\pi(a \mid x)p(r \mid x, a)}[r] = \mathbb{E}_{p(x)\pi(a \mid x)}[q(x, a)] \notag
\end{align}
The algorithm performance $V(\pi)$ represents the expected reward obtained from deploying it in a real-world environment.
In a video recommendation task, for example, the expected number of user clicks or the viewing time obtained from deploying the algorithm constitutes the algorithm's performance. 

In this paper, we consider an A/B testing scenario comparing two different algorithms $\pi_A$ and $\pi_B$.
Let the A/B testing data collected from each algorithm be $\mathcal{D}_A = \{(x_i^A, a_i^A, r_i^A)\}_{i=1}^{n^A}$ and $\mathcal{D}_B = \{(x_i^B, a_i^B, r_i^B)\}_{i=1}^{n^B}$.
The problem we tackle is constructing a performance difference estimator $\Delta\hat{V}$ using the A/B testing data such that its sign matches that of the true performance difference $V(\pi_A) - V(\pi_B)$.
To measure how inaccurately the difference estimator selects the algorithm, we define the error rate in algorithm selection as the primary metric:
\begin{equation}
    error\text{-}rate(\Delta\hat{V}) = \mathbb{E}_{p(\mathcal{D}_A, \mathcal{D}_B)}\bigg[\mathbb{I}\big\{\text{sign}\big(\Delta V) \neq \text{sign}\big(\Delta\hat{V}\big)\big\}\bigg]
\end{equation}
Here, $sign(\cdot)$ denotes the sign function, which returns 1 if the argument is positive, -1 if it is negative, and 0 if it is zero.
$\mathbb{I}\{\cdot\}$ is the indicator function, which returns 1 when the condition is satisfied and 0 otherwise.
This rate is defined as the probability that the estimator incorrectly selects the truly worse algorithm based on a finite sample.
For instance, an error rate of 30\% implies that there is a 30\% chance that the estimator will wrongly select the worse algorithm given the logged data.

The algorithm selection error rate of a difference estimator $\Delta\hat{V}$ depends mainly on the following two factors.\\ \\
\noindent \textbf{Bias}: This measures how far the expected value of the difference estimator $\Delta\hat{V}$ is from the true difference $\Delta V$.
\begin{equation}
    Bias\big[\Delta\hat{V}\big] = \mathbb{E}_{p(\mathcal{D}_A, \mathcal{D}_B)}\big[\Delta V - \Delta\hat{V} \big]
\end{equation}

\noindent \textbf{Variance}: This measures the variation of the estimates derived from the difference estimator $\Delta\hat{V}$.
\begin{equation}
    Var\big[\Delta\hat{V}\big] = \mathbb{E}_{p(\mathcal{D}_A, \mathcal{D}_B)}\bigg[\bigg(\Delta\hat{V} - \mathbb{E}_{p(\mathcal{D}_A, \mathcal{D}_B)}\big[\Delta\hat{V}\big]\bigg)^2 \bigg]
\end{equation}
Moreover, taking the AVG-based difference estimator $\Delta\hat{V}_{AVG} = \hat{V}_{AVG}(\pi_A) - \hat{V}_{AVG}(\pi_B)$ as an example, its variance can be decomposed as follows.
\begin{align}
    &Var\big[\hat{V}_{AVG}(\pi_A) - \hat{V}_{AVG}(\pi_B)\big] \label{eq:difference-var} \\
    &= Var\big[\hat{V}_{AVG}(\pi_A)\big] + Var\big[\hat{V}_{AVG}(\pi_B)\big] -2Cov\big[\hat{V}_{AVG}(\pi_A), \hat{V}_{AVG}(\pi_B)\big] 
     \notag
\end{align}
This equation shows that when a correlation (covariance) exists between the two performance estimators, $\hat{V}(\pi_A)$ and $\hat{V}(\pi_B)$, it affects the total variance. In particular, if a positive correlation exists between the estimators, the total variance will be less than the sum of the individual variances.

\subsection{Existing Methods}
\subsubsection{\textbf{The AVG estimator in A/B testing}}
The most widely used approach for algorithm selection in A/B testing is the AVG estimator \cite{Kohavi2012-wk, Kohavi2013-rx, Wang2019-do}. This method independently estimates the performance of each algorithm by computing the sample mean of the observed rewards $\{r_i\}_{i=1}^{n}$ for each group and then compares these estimates to select the better algorithm. Specifically, AVG selects an algorithm based on the sign of the following difference estimator.
\begin{align}
    \Delta\hat{V}_{AVG} = \frac{1}{n^A}\sum_{i=1}^{n^A}r_i^A - \frac{1}{n^B}\sum_{i=1}^{n^B}r_i^B
    \label{eq:avg-difference}
\end{align}
Since this difference estimator is an unbiased estimator of the true performance difference (i.e., $\mathbb{E}_{p(\mathcal{D}_A, \mathcal{D}_B)}[\Delta\hat{V}_{AVG}] = \Delta V$), the algorithm selection error rate is governed entirely by its variance.
In A/B testing, since data collected for algorithms A and B are independent by design, the two performance estimators $\hat{V}_{AVG}(\pi_A; \mathcal{D}_A)$ and $\hat{V}_{AVG}(\pi_B; \mathcal{D}_B)$ are uncorrelated (i.e., their covariance is zero). This property is already well-illustrated in Figure~\ref{fig:correlation}, where the scatter plot corresponding to AVG exhibits a clear circular shape. Due to the absence of correlation, the covariance term in Eq.~\ref{eq:difference-var} vanishes, and the variance of the overall AVG difference estimator reduces to the sum of the individual variances
\begin{align}
    Var\big[\Delta\hat{V}_{AVG}\big] = Var\big[\hat{V}_{AVG}(\pi_A)\big] + Var\big[\hat{V}_{AVG}(\pi_B)\big] \notag
\end{align}

To establish a theoretical baseline for comparison with our proposed method, we analyze this variance term in greater detail.
Assuming for simplicity that the sample sizes are equal (i.e., $n=n^A=n^B$), we can expand the variance of AVG as
\begin{align}
    &Var\big[\Delta\hat{V}_{AVG}\big]\notag\\
    &= \frac{1}{n}\bigg\{\mathbb{E}_{p(x)}\bigg[\sum_{a \in \mathcal{A}}\big(\pi_A(a \mid x) + \pi_B(a \mid x)\big)\big(q^2(x, a) + \sigma^2(x, a)\big)\bigg]\notag\\
    &\hspace{5cm}- V^2(\pi_A) - V^2(\pi_B)\bigg\}
    \label{eq:AVG-variance}
\end{align}
This expression consists exclusively of independent variance components, explicitly highlighting the absence of any covariance term that could otherwise reduce the probability of critical selection errors. In other words, AVG offers no mechanism to reduce variance by exploiting correlation between the performance estimates of algorithms A and B, which constitutes a fundamental limitation for algorithm selection. This limitation directly explains why AVG underperformed an estimator based on offline data in our pre-experiment, despite having access to ideal A/B testing data.

\subsubsection{\textbf{Standard Offline Evaluation Method}}
Although the primary objective of this work is to develop a novel estimator for A/B testing that outperforms AVG in algorithm selection, the pre-experiment in Section~1 shows that a standard offline evaluation method, namely IPS, achieves a lower selection error rate than AVG even when A/B testing data are available.
Since our proposed approach explicitly leverages A/B testing data collected from A/B testing, it is therefore naturally required to outperform offline evaluation methods that rely exclusively on offline data.
Accordingly, we include offline evaluation methods as an important baseline and describe them in detail below.

The IPS estimator is one of the most widely used offline evaluation methods \cite{Horvitz1952-ws, Su2020-av}.
IPS estimates the performance of a new algorithm by re-weighting the rewards observed under a logging algorithm that differs from the new one.
Formally, given data $\mathcal{D}_A$ collected by the logging algorithm $\pi_A$, the IPS estimator for a new algorithm $\pi_B$ is defined as
\begin{align}
    \hat{V}_{IPS}(\pi_B; \mathcal{D}_A) \notag
    &= \frac{1}{n^A}\sum_{i=1}^{n^A}\frac{\pi_B(a_i^A \mid x_i^A)}{\pi_A(a_i^A \mid x_i^A)}r_i^A = \frac{1}{n^A}\sum_{i=1}^{n^A}w(x_i^A, a_i^A)r_i^A \notag
\end{align}
where $\pi_A$ and $\pi_B$ denote the logging and new algorithm, respectively, and $w(x,a)=\pi_B(a\mid x)/\pi_A(a\mid x)$ is the importance weight.

To perform algorithm selection, we construct an IPS-based difference estimator. Specifically, we select the algorithm based on the sign of
\begin{align}
    \Delta\hat{V}_{IPS}
    &= \hat{V}_{IPS}(\pi_A; \mathcal{D}_A) - \hat{V}_{IPS}(\pi_B; \mathcal{D}_A) \notag\\
    &= \frac{1}{n^A}\sum_{i=1}^{n^A}\frac{\pi_A(a_i^A \mid x_i^A) - \pi_B(a_i^A \mid x_i^A)}{\pi_A(a_i^A \mid x_i^A)}r_i^A
    \label{eq:ips-difference}
\end{align}
where algorithm $A$ is selected if $\Delta\hat{V}_{IPS} > 0$, and algorithm $B$ otherwise.

Under the common support assumption\footnote{Common support assumption \cite{Sachdeva2020-ha}: the logging policy $\pi_A$ has common support for the new policy $\pi_B$ if $\pi_B(a\mid x) > 0$ implies $\pi_A(a\mid x) > 0$ for all $a \in \mathcal{A}$ and $x \in \mathcal{X}$.}, the IPS-based difference estimator is unbiased.
\begin{align}
    \mathbb{E}_{p(\mathcal{D}_A)}[\Delta\hat{V}_{IPS}] = \Delta V \notag
\end{align}
Furthermore, the variance of IPS is given by
\begin{align}
    &Var\big[\hat{V}_{IPS}(\pi_A;\mathcal{D}_A) - \hat{V}_{IPS}(\pi_B;\mathcal{D}_A)\big] \notag \\
    &= Var\big[\hat{V}_{IPS}(\pi_A;\mathcal{D}_A)\big] + Var\big[\hat{V}_{IPS}(\pi_B;\mathcal{D}_A)\big] \notag \\
    & ~~~~~~~~~~~~~~~~~~~~~~~~~~-2Cov\big[\hat{V}_{IPS}(\pi_A;\mathcal{D}_A), \hat{V}_{IPS}(\pi_B;\mathcal{D}_A)\big] \notag
\end{align}
Because both estimators $\hat{V}_{IPS}(\pi_A;\mathcal{D}_A)$ and $\hat{V}_{IPS}(\pi_B;\mathcal{D}_A)$ are computed using the same offline data $\mathcal{D}_A$, they are generally correlated rather than independent, as empirically illustrated in Figure~\ref{fig:correlation}.
When this correlation is positive, the covariance term reduces the total variance below the sum of the individual variances \cite{Jeunen2024-vb}.
This variance reduction effect explains why the estimator based on offline data can achieve a lower selection error rate than AVG, even without A/B testing.

To make this mechanism explicit and to establish a theoretical baseline for comparison with our proposed method, we expand the variance of the IPS-based difference estimator. Assuming equal sample sizes $n=n^A=n^B$ for simplicity, we obtain
\begin{align}
    &Var\big[\Delta\hat{V}_{IPS}\big]\notag\\
    &= \frac{1}{n}\bigg\{\mathbb{E}_{p(x)}\bigg[\sum_{a \in \mathcal{A}}\frac{\big(\pi_A(a \mid x) - \pi_B(a \mid x)\big)^2}{\pi_A(a \mid x)}\big(q^2(x, a) + \sigma^2(x, a)\big)\bigg] \notag\\
    &\hspace{5cm}- \big(V(\pi_A) - V(\pi_B)\big)^2\bigg\}
    \label{eq:IPS-variance}
\end{align}
Unlike the AVG-based difference estimator, expanding the squared term in this expression produces a negative cross term, which directly accounts for the reduction in total variance.
This variance reduction mechanism is intrinsic to offline evaluation methods, including the Doubly Robust estimator \cite{Su2020-av}, and arises from estimating the performance of multiple algorithms using shared data.

Despite this advantage, their reliance on importance weights $w(x,a)=\pi_B(a\mid x)/\pi_A(a\mid x)$ \cite{Su2020-av, Swaminathan2015-dj, Metelli2021-ym} can lead to severe variance inflation when the similarity between the compared algorithms is low.
Such variance inflation can negate the benefits of positive covariance, making reliable algorithm selection challenging in practice.
\section{A Proposed Method}
This section introduces a novel estimator, MID, and analyzes its advantages through a bias--variance decomposition.

\subsection{Formulation of the MID estimator}
In A/B testing, the AVG estimators \(\hat{V}_{AVG}(\pi_A; \mathcal{D}_A)\) and \(\hat{V}_{AVG}(\pi_B; \mathcal{D}_B)\) are uncorrelated, as empirically demonstrated in the pre-experiment.
Consequently, the AVG-based difference estimator does not benefit from the variance reduction effect described in Eq.~\ref{eq:difference-var}, which shows that inducing positive correlation can reduce the total variance below the sum of individual variances.
In contrast, offline evaluation methods unintentionally induce such positive correlation as a byproduct of estimating the performance of multiple algorithms using shared data. Motivated by this insight, we propose the MID estimator, which intentionally induces positive correlation using initially independent A/B testing data in order to maximize the variance reduction effect and improve the algorithm selection error rate. To intentionally induce positive correlation in an A/B testing setting, we exploit the core mechanism underlying offline evaluation, namely the use of offline data to estimate the performance of multiple algorithms.
To transfer this structure to the A/B testing datasets \(\mathcal{D}_A\) and \(\mathcal{D}_B\), we introduce a middle algorithm \(\pi_M\) that serves as an intermediary between algorithms \(\pi_A\) and \(\pi_B\).
We then estimate the performance difference between \(\pi_A\) and \(\pi_B\) in a stepwise manner via \(\pi_M\).
This construction enables MID to effectively leverage the variance reduction effect induced by positive correlation, resulting in a lower error rate in A/B testing.

MID begins by decomposing the true performance difference into the following two components (note that we will later discuss how we should obtain the middle algorithm $\pi_M$).
\begin{equation}
    \Delta V = {\color{red}\underline{\textcolor{black}{V(\pi_A) - V(\pi_M)}}} + {\color{blue}\underline{\textcolor{black}{V(\pi_M) - V(\pi_B)}}} \notag
\end{equation}
We estimate the first term using an offline evaluation method with log data \(\mathcal{D}_A\) and the second term using log data \(\mathcal{D}_B\).
In what follows, we instantiate MID using IPS, although more advanced methods, such as Doubly Robust estimation, can be incorporated in a straightforward manner. Formally, we define the MID estimator as
\begin{align}
    \Delta \hat{V}_{MID} = &\hat{V}_{IPS}(\pi_A;\mathcal{D}_A) - \hat{V}_{IPS}(\pi_M;\mathcal{D}_A) \notag\\
    &+ \hat{V}_{IPS}(\pi_M;\mathcal{D}_B) - \hat{V}_{IPS}(\pi_B;\mathcal{D}_B) \notag
    \label{eq:mid-estimator}
\end{align}
Crucially, MID uses the same log data to estimate the difference term, $\hat{V}_{IPS}(\pi_A;\textcolor{red}{\mathcal{D}_A})- \hat{V}_{IPS}(\pi_M;\textcolor{red}{\mathcal{D}_A})$.
As a result, the pair of estimators induces a positive correlation (in the next section, we will empirically demonstrate that our estimator can indeed intentionally introduce such a positive correlation between estimates).
The same applies to the third and fourth terms, which use the common log data $\mathcal{D}_B$.
Consequently, a positive correlation occurs even when estimating performance differences with A/B testing data ($\mathcal{D}_A$ and $\mathcal{D}_B$), allowing us to expect a variance reduction effect.

A key strength of MID lies in the ability to explicitly design the middle algorithm \(\pi_M\) according to our objectives.
As discussed in the previous section, offline evaluation methods, which constitute a core component of MID, can suffer from severe variance inflation when the similarity between the compared algorithms is low.
By designing \(\pi_M\) to remain highly similar to both \(\pi_A\) and \(\pi_B\), MID ensures that each stepwise comparison includes only similar algorithms, thereby mitigating variance explosion. Guided by this principle, we derive the optimal middle algorithm by minimizing the total variance of MID, with the goal of further reducing algorithm selection errors.

As a first step, we express the variance of MID by the equation below (assuming $n=n^A=n^B$ here for brevity).
\begin{align}
    &Var[\Delta\hat{V}_{MID}] \notag\\
    &= \frac{1}{n}\bigg\{\mathbb{E}_{p(x)}\bigg[\sum_{a \in \mathcal{A}}\bigg(\frac{\Delta\pi_A^2(a \mid x)}{\pi_A(a \mid x)}+\frac{\Delta\pi_B^2(a \mid x)}{\pi_B(a \mid x)}\bigg)\big(q^2(x, a) + \sigma^2(x, a)\big)\bigg]\notag\\
    &\;\;\;\;\;\;\;\;\;\;- \big(V(\pi_A) - V(\pi_M)\big)^2 - \big(V(\pi_B) - V(\pi_M)\big)^2\bigg\},\\
    &\;\;\;\;\;\;\;\;\;\;\Delta\pi_A^2(a \mid x) = \big(\pi_A(a \mid x) - \pi_M(a \mid x)\big)^2,\notag\\
    &\;\;\;\;\;\;\;\;\;\;\Delta\pi_B^2(a \mid x) = \big(\pi_B(a \mid x) - \pi_M(a \mid x)\big)^2\notag
    \label{eq:MID-variance}
\end{align}
We decompose this variance into a reward term $\mathbb{E}_{p(x)} [\sum_{a \in \mathcal{A}}\bigg(\frac{\Delta\pi_A^2(a \mid x)}{\pi_A(a \mid x)} + \frac{\Delta\pi_B^2(a \mid x)}{\pi_B(a \mid x)}\bigg)\big(q^2(x, a) + \sigma^2(x, a)\big)]$ and a performance term $\big(V(\pi_A) - V(\pi_M)\big)^2 + \big(V(\pi_M) - V(\pi_B)\big)^2$.
Because the performance term depends on unknown true policy values, we focus on minimizing the reward term.
By differentiating the reward term with respect to \(\pi_M(a \mid x)\) and setting the derivative to zero, we obtain the optimal middle algorithm.
The following theorem summarizes the result.
\begin{thm}
    The optimal middle algorithm $\pi_{M^*}$ that minimizes the reward term of the variance of MID is given by
    \begin{equation}
        \pi_{M^*}(a \mid x) = \frac{2 \pi_A(a \mid x)\pi_B(a \mid x)}{\pi_A(a \mid x) + \pi_B(a \mid x)} \notag
    \end{equation}
    \label{thm:mid-policy}
\end{thm}
Moreover, when MID uses the optimal middle algorithm \(\pi_{M^*}\), the estimator is unbiased.
The following theorem provides a sufficient condition.
\begin{thm}
    MID is unbiased provided that the middle algorithm $\pi_M$ satisfies the following condition
    \begin{equation}
        \pi_A(a \mid x) = 0 \oplus \pi_B(a \mid x) = 0 \Rightarrow\pi_M(a \mid x) = 0 \notag
    \end{equation}
    \label{thm:cond-unbiased}
    \vspace{-5mm}
\end{thm}

Using the optimal middle algorithm $\pi_{M^*}$ derived in Theorem~\ref{thm:mid-policy}, we arrive at the following formulation of MID.
\begin{align}
    \Delta \hat{V}_{MID} = &\frac{1}{n^A}\sum_{i=1}^{n^A}\frac{\pi_A(a_i^A \mid x_i^A) - \pi_B(a_i^A \mid x_i^A)}{\pi_A(a_i^A \mid x_i^A) + \pi_B(a_i^A \mid x_i^A)}r_i^A \notag\\
    &+\frac{1}{n^B}\sum_{i=1}^{n^B}\frac{\pi_A(a_i^B \mid x_i^B) - \pi_B(a_i^B \mid x_i^B)}{\pi_A(a_i^B \mid x_i^B) + \pi_B(a_i^B \mid x_i^B)}r_i^B
    \label{eq:mid-difference}
\end{align}

This final form of the MID estimator enables more accurate algorithm selection than AVG by reducing variance through the intentional induction of positive correlation between estimators, as well as through the principled and theoretically grounded design of the middle algorithm $\pi_M$.\footnote{The proofs of Theorems~\ref{thm:mid-policy} and~\ref{thm:cond-unbiased} are described in Appendix~\ref{app:proof}.}

\subsection{Theoretical Comparison of Variance}
In this section, we theoretically compare the variance of MID with that of the AVG-based and IPS-based difference estimators.

\subsubsection*{\textbf{Comparison against the AVG-based Difference Estimator}}
Recall the variance of the $\Delta\hat{V}_{AVG}$ derived in Eq. \ref{eq:AVG-variance}.
Focusing on the reward term (the component depending on $q^2(x, a)$ and $\sigma^2(x, a)$), the variance of the $\Delta\hat{V}_{AVG}$ can be expressed as follows.
\begin{equation}
\mathrm{R}(x, a;\Delta\hat{V}_{AVG}) = \pi_A + \pi_B \notag
\end{equation}
where we omit the arguments $(a \mid x)$ and denote $\pi(a \mid x)$ simply as $\pi$ for notational brevity.

In contrast, we analyze the reward term of the $\Delta\hat{V}_{MID}$ using the optimal middle algorithm $\pi_{M^*}$.
The reward term is derived as
\begin{align*}
    R(x, a; \Delta\hat{V}_{MID})
    &= \frac{\big(\pi_A - \pi_{M^*}\big)^2}{\pi_A} + \frac{\big(\pi_B - \pi_{M^*}\big)^2}{\pi_B}\\
    &= \pi_A + \pi_B -\frac{4\pi_A\pi_B}{\pi_A + \pi_B}\\
    & \le \pi_A + \pi_B = R(x, a; \Delta\hat{V}_{AVG})
\end{align*}
Therefore, by substituting $\pi_{M^*}$ into the variance equation, we can mathematically demonstrate that the reward term of the $\Delta\hat{V}_{MID}$ is consistently smaller than or equal to that of the $\Delta\hat{V}_{AVG}$.
This inequality demonstrates that our method achieves variance reduction by inducing positive correlation in the presence of stochastic rewards, a property that is absent from conventional A/B testing approaches based on the AVG estimator.

Finally, we examine the performance term, which serves as a variance reduction component, and compare its magnitude across estimators.
For $\Delta\hat{V}_{\mathrm{AVG}}$ and $\Delta\hat{V}_{\mathrm{MID}}$, we define the corresponding performance terms as follows.
\begin{align}
    \mathrm{Per}(\Delta\hat{V}_{AVG}) &= V(\pi_A)^2 + V(\pi_B)^2 \notag\\
    \mathrm{Per}(\Delta\hat{V}_{MID}) &= \big(V(\pi_A) - V(\pi_M)\big)^2 + \big(V(\pi_B) - V(\pi_M)\big)^2 \notag
\end{align}
Our optimal middle algorithm $\pi_{M^*}$ does not explicitly minimize this component, as it depends on unknown true policy values.
From a theoretical standpoint, $\Delta\hat{V}_{\mathrm{AVG}}$ is superior in this specific respect, since it benefits from a larger subtraction term.
However, in practical scenarios, performance improvements are typically incremental, that is, the difference between $V(\pi_A)$ and $V(\pi_B)$ is small \cite{Kohavi2013-rx, Bojinov2022-ah}.
As a result, the magnitude of this performance term often becomes negligible, and its contribution to the overall variance is minimal.

\subsubsection*{\textbf{Comparison against the IPS-based Difference Estimator}}
Next, we consider the IPS-based difference estimator.
By expanding the reward and performance terms in Eq.~\ref{eq:IPS-variance} in the same manner as in the previous analysis, we obtain
\begin{align}
    R(x, a; \Delta\hat{V}_{IPS}) = \pi_A + \frac{\pi_B^2}{\pi_A} - 2\pi_B \notag\\
    \text{Per}(\Delta\hat{V}_{IPS}) = \big(V(\pi_A) - V(\pi_B)\big)^2 \notag
\end{align}
This derivation reveals that the term $\pi_B^2 / \pi_A$ leads to variance explosion when the logging algorithm $\pi_A$ assigns a low probability to actions that the target algorithm $\pi_B$ selects with high probability.
This mathematical structure explains the instability of offline evaluation when the similarity between algorithms is low.
In contrast, as derived above, the reward term of MID is bounded as $R(x, a; \Delta\hat{V}_{MID}) \le \pi_A + \pi_B$.
By mediating the estimation through the optimal middle algorithm $\pi_{M^*}$, our method effectively removes the risk of variance explosion and ensures stable comparison even when the algorithms are dissimilar.
\section{Empirical Evaluation}
This section validates the effectiveness of MID in various A/B testing environments constructed from a real-world recommendation dataset. Note that our code to replicate the experiment results is available on \url{https://kdd2026-mid.github.io/}.

\subsection{Setup}
\subsubsection*{\textbf{Dataset}}
In this experiment, we use the KuaiRec dataset,\footnote{https://kuairec.com/} a real-world video viewing dataset collected from the recommendation system on the video-sharing mobile app Kuaishou.
This dataset consists of 1,411 users, 3,327 items, and 4,676,570 interactions, providing the watch ratio (the percentage of total duration viewed) for the entire user-item matrix.
The availability of this fully-observed matrix allows us to calculate the exact ground-truth algorithm performance.
This feature is particularly crucial because calculations of the performance of algorithms based on standard partially-observed datasets often misidentify the true performance ranking of algorithms~\cite{Gao2022-pb}.
To our knowledge, KuaiRec is the only public real-world dataset that satisfies this requirement.
For this reason, we chose it to simulate an algorithm selection task.
We map the components of the KuaiRec dataset to the Contextual Bandit formulation where the context $x$ is user information, the action $a$ is the recommended video, and the watch ratio is used as the reward function $q(x, a)$.

\subsubsection*{\textbf{Building the Video Recommendation Algorithm}}
\begin{figure}
    \includegraphics[width=0.85\linewidth]{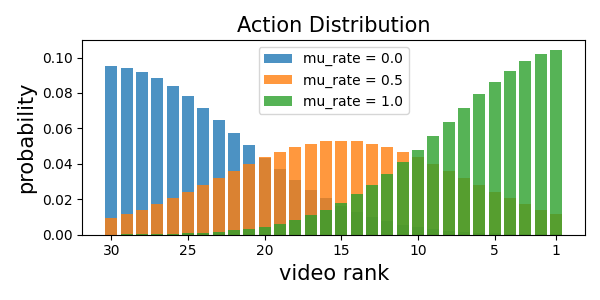}
    \vspace{-4mm}
    \caption{Discrete normal distributions with varying peaks}
    \vspace{-3mm}
    \label{fig:normal-policy}
\end{figure}
To build video recommendation algorithms for A/B testing simulations, we follow the steps below.
First, given a user context $x$, we retrieve the true expected reward $q(x, a)$ for all candidate videos $a$ and generate a ranking based on these scores.
Next, we assign action probabilities according to the rank, based on a predefined discrete normal distribution $\pi(a_{rank}; \mu_{rate})$
\begin{align*}
    \pi(a_{rank}; \mu_{rate}) = \frac{1}{\sqrt{2\pi}\sigma}\exp\bigg(-\frac{1}{2}\bigg(\frac{a_{rank}-|\mathcal{A}| \cdot\mu_{rate}}{\sigma}\bigg)^2\bigg)
\end{align*}
Figure \ref{fig:normal-policy} illustrates the discrete normal distributions for specific values of the parameter $\mu_{rate}$.
By adjusting the parameter $\mu_{rate}$, the constructed algorithm $\pi(a \mid x; \mu_{rate})$ focuses its recommendations on videos of specific tasks.
Specifically, setting $\mu_{rate}=1.0$ makes the algorithm likely to recommend the videos with the highest expected reward; $\mu_{rate}=0.5$ targets videos with the median expected reward; and $\mu_{rate}=0.0$ prioritizes the videos with the lowest expected reward. 
Consequently, by varying the parameters $\mu_A$ and $\mu_B$, we can simultaneously control the behavior similarity and the performance difference between algorithms $\pi_A$ and $\pi_B$, which enables us to adjust the difficulty of the algorithm selection tasks.

\subsubsection*{\textbf{Building the A/B testing Environment}}
To compare the error rate of our proposed method against existing methods, we build an A/B testing environment using recommendation algorithms $\pi_A$ and $\pi_B$ to generate log data.
The data generation process by algorithms $\pi_A$ and $\pi_B$ follows these specific steps:
\begin{enumerate}
    \item \textbf{User visit:} When a user visits the service, the recommendation system retrieves the user information, such as demographics and past viewing history, as the context $x$.
    \item \textbf{Random assignment (A/B testing):} The recommendation system assigns the user to either Group A or Group B.
    \item \textbf{Video recommendation:} If the user is assigned to Group A, the algorithm $\pi_A$ selects and presents a video predicted to suit the user's interests based on the user information.
    The same applies to Group B.
    The recommended video corresponds to the action $a$.
    \item \textbf{Recording results:} We observe and record the watch ratio as the reward $r$, which is $q(x, a)$ with white noise.
\end{enumerate}
By repeating this procedure for a given sample size, we generate log data $\mathcal{D}_A$ and $\mathcal{D}_B$ corresponding to the two algorithms and execute algorithm selection with the proposed and baseline estimators.
Finally, the metrics reported in this paper are calculated based on the estimator selection results of 10,000 independent trials.

\subsubsection*{\textbf{Compared Estimators for Algorithm Selection}}
In this experiment, we compare the error rate of three estimators in algorithm selection.
They include the AVG-based difference estimator (Eq. \ref{eq:avg-difference}) and the MID estimator (Eq. \ref{eq:mid-difference}), both using A/B testing data $\mathcal{D}_A$ and $\mathcal{D}_B$, and the IPS-based difference estimator (Eq. \ref{eq:ips-difference}), which uses only offline data (specifically $\mathcal{D}_A$ in this experiment).

\subsubsection*{\textbf{Evaluation Metrics}}
In this experiment, we evaluate each estimator using the following three metrics.
\begin{itemize}
    \item \textbf{Error rate}, which measures the ability to select the better algorithm and is the primary objective of this study.
    \item \textbf{Variance}, which quantifies the dispersion or instability in estimating the performance difference.
    Since all estimators in this paper are theoretically unbiased, variance is a key component of algorithm selection error rate.
    \item \textbf{Statistical power}, which is essential in A/B testing to avoid missing valid improvements and to avoid opportunity loss.
    We verify this by applying Welch's t-test to the estimates from the MID, AVG, and IPS estimators, comparing their rejection rates (power) for different significance levels.
\end{itemize}

\subsection{Results}
In this section, we present the experimental results comparing our MID estimator against the existing estimators.

\begin{figure*}
    \centering
    \includegraphics[width=0.95\linewidth]{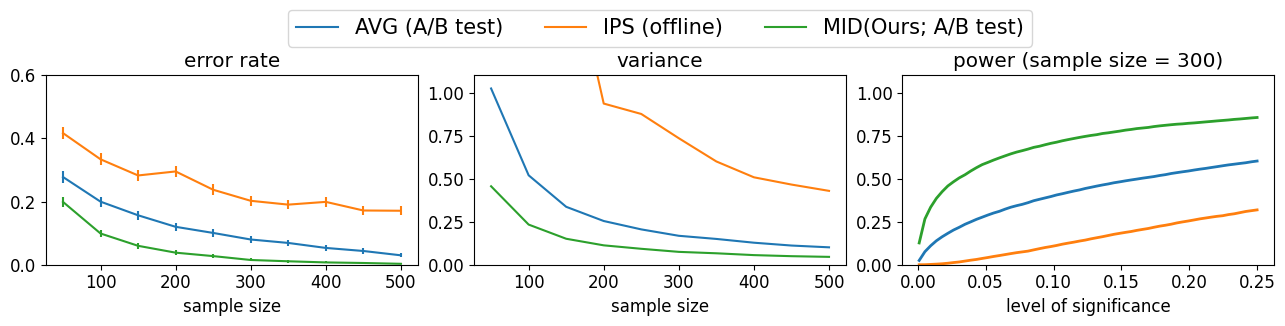}
    \vspace{-3mm}
    \caption{Comparison of the estimator's error rate, variance, and statistical power with varying sample sizes}
    \vspace{-3mm}
    \label{fig:result-sample-size}
\end{figure*}
\begin{figure*}
    \centering
    \includegraphics[width=0.95\linewidth]{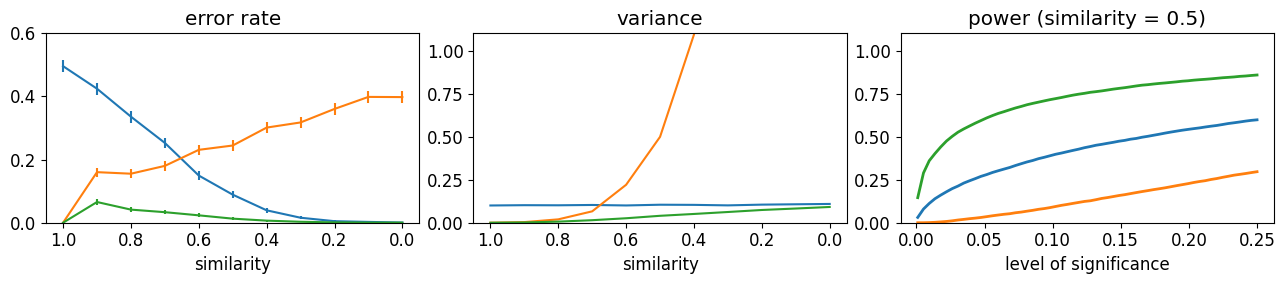}
    \vspace{-3mm}
    \caption{Comparison of the estimator's error rate, variance, and statistical power with varying levels of algorithm similarity}
    \vspace{-3mm}
    \label{fig:result-similar}
\end{figure*}

First, Figure \ref{fig:result-sample-size} illustrates the error rates of the estimators as we vary the sample size from 50 to 500 when the algorithm similarity is 0.5.
In the context of A/B testing, the sample size corresponds to the experiment duration or the number of participating users.
Consequently, a method capable of achieving a lower error rate with a small sample size enables rapid and cost-effective algorithm selection.
This experiment investigates the extent to which our MID estimator improves sample efficiency compared to conventional methods.
The results show that the sample efficiency of MID is better than that of existing methods.
Specifically, \textbf{while AVG using A/B testing data requires a sample size of 350-400 to achieve a 5\% error rate, MID achieves a 5\% error rate with a sample size of 100--200. This indicates that the MID requires only one-half to one-quarter of the sample size to achieve the same error rate as conventional A/B testing.}
Our results further demonstrate that MID consistently exhibits substantially higher statistical power than AVG using A/B testing data and IPS using offline data.
This efficiency gain is attributed to the positive correlation induced by MID, which effectively reduces the variance when estimating the performance difference of the algorithms.

Figure \ref{fig:result-similar} illustrates the impact of the similarity between the algorithms $\pi_A$ and $\pi_B$ on the error rate of each estimator.
In this experiment, we verify that MID overcomes the variance inflation caused by low algorithm similarity.
Typically, IPS becomes unstable when the similarity between algorithms is low because the importance weight $w(x, a)$ explodes.
In contrast, AVG remains stable regardless of algorithm similarity, but it fails to benefit from variance reduction via the correlation mechanism.
This experiment demonstrates whether the proposed method can overcome this trade-off and consistently achieve a lower error rate than existing methods, regardless of algorithm similarity.
Specifically, we fix the peak-adjustment parameter of algorithm $\pi_A$ at $\mu_{rate_A}=0.0$ and sweep the parameter $\mu_{rate_B}$ of algorithm $\pi_B$ from 0.0 to 1.0 on the x-axis.
This setup corresponds to algorithm improvement testing in real-world recommendation system development: minor modifications result in high algorithm similarity, while significant changes lead to low similarity. 
As shown in the results, the proposed method achieves a lower error rate than the AVG and IPS under all experimental conditions.
It outperforms IPS even in high-similarity scenarios where offline methods are generally effective, while it also outperforms AVG in low-similarity scenarios.
This result is attributed to the optimal design of the middle algorithm $\pi_{M^*}$ used in the proposed method.
In scenarios with low algorithm similarity, where IPS fails to avoid variance explosion, our method estimates the performance difference between more similar pairs (e.g., $\pi_A$ and $\pi_{M^*}$).
Thus, the optimal middle algorithm $\pi_{M^*}$ enables stable estimation by our method even when algorithm similarity is low.

Finally, we investigate the inner workings of MID and how it reduces the selection error rate. 
We begin by demonstrating whether the positive correlation, the key factor for reducing the critical error, is generated as intended.
MID is designed to induce a positive correlation between the following two pairs of estimators
\begin{enumerate}
    \item $\hat{V}_{IPS}(\pi_A;\mathcal{D}_A)$ and $\hat{V}_{IPS}(\pi_M;\mathcal{D}_A)$
    \item $\hat{V}_{IPS}(\pi_B;\mathcal{D}_B)$ and $\hat{V}_{IPS}(\pi_M;\mathcal{D}_B)$
\end{enumerate}
Figure \ref{fig:proposed-correlation} plots the estimates for these pairs, with the left and right sides corresponding to dataset $\mathcal{D}_A$ and $\mathcal{D}_B$, respectively. 
The results indicate that both pairs of estimators show a clear positive correlation and that the variance reducing mechanism of our proposed method works as expected.
\begin{figure}
    \centering
    \includegraphics[width=0.95\linewidth]{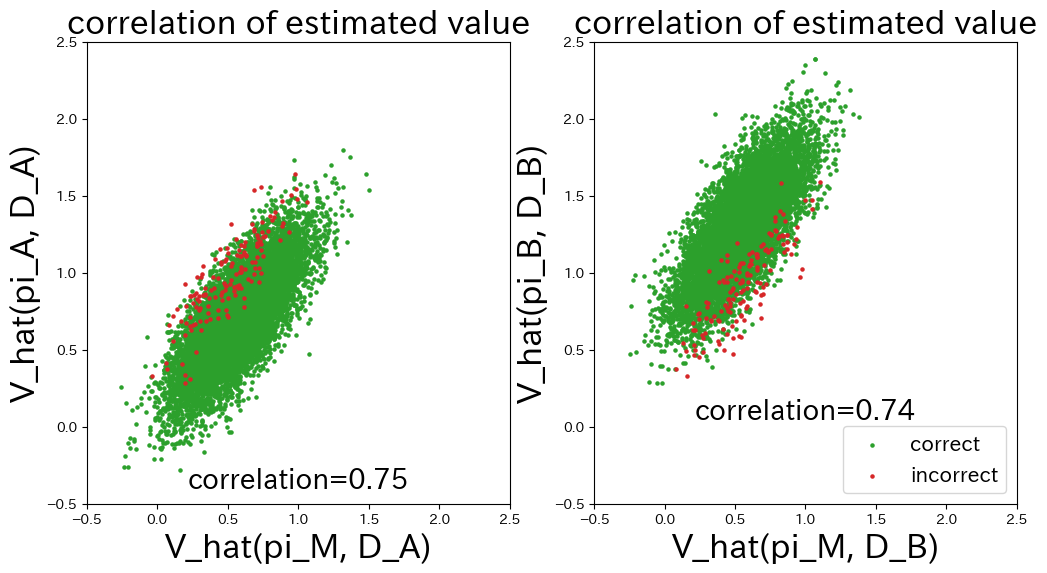}
    \vspace{-3mm}
    \caption{Scatter plots of MID estimates for $\pi_A$ and $\pi_B$ when the sample size is 300.}
    \vspace{-1mm}
    \label{fig:proposed-correlation}
\end{figure}

\begin{figure}
    \centering
    \includegraphics[width=0.95\linewidth]{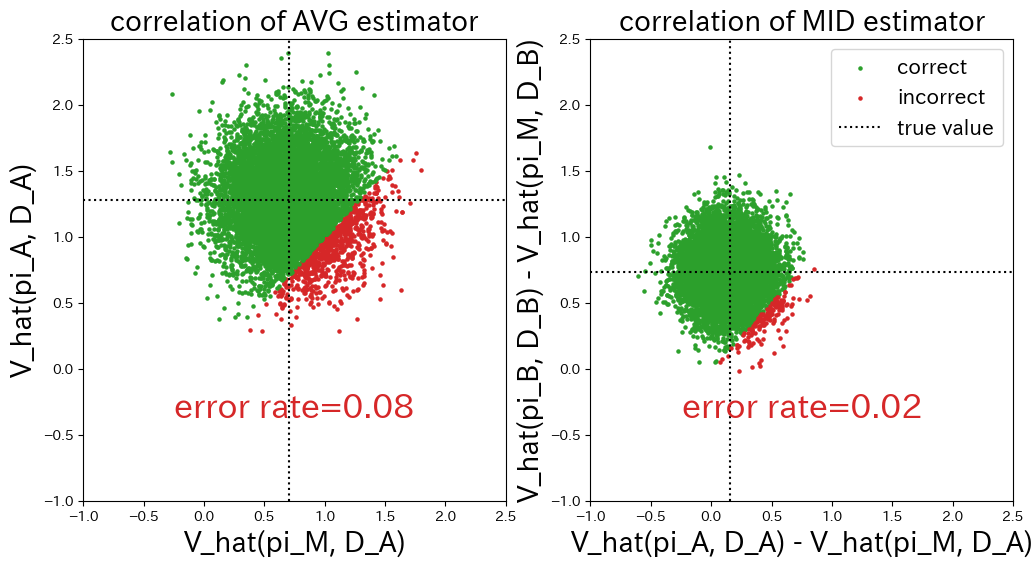}
    \vspace{-3mm}
    \caption{Comparison of scatter plots: AVG vs MID.}
    \vspace{-3mm}
    \label{fig:correlation-avg-and-mid}
\end{figure}

Next, we verify whether this mechanism reduces critical errors and thereby contributes to improving the selection error rate.
Figure \ref{fig:correlation-avg-and-mid} plots the final comparative estimates for AVG and MID.
The left panel shows the behavior of AVG, where the axes correspond to the performance estimates of $\pi_A$ and $\pi_B$.
The scatter plot appears circular due to the independence of the datasets $\mathcal{D}_A$ and $\mathcal{D}_B$ used for its estimation.
In contrast, the right panel shows MID, plotting $\hat{V}_{IPS}(\pi_A;\mathcal{D}_A)-\hat{V}_{IPS}(\pi_M;\mathcal{D}_A)$ and $\hat{V}_{IPS}(\pi_B;\mathcal{D}_B)-\hat{V}_{IPS}(\pi_M;\mathcal{D}_B)$.
This plot is also circular, as the estimators on the respective axes use independent datasets.
Comparing the two, we observe a shift in the center points of the scatter plots, but crucially, the spread of the scatter plot is much tighter for MID than for AVG.
This tighter spread is a direct result of the variance reduction effect caused by the positive correlation that we introduced in the A/B testing analysis.
These observations well demonstrate that the mechanism of intentionally inducing a positive correlation in our method to reduce variance and algorithm selection error rate works as designed.
\section{Conclusion and Future Work}
This paper first showed that standard A/B testing based solely on sample average estimation is often inferior to offline evaluation in terms of algorithm selection error rate.
This limitation arises because the AVG estimator used in A/B testing does not exploit the beneficial positive correlation that naturally emerges when shared data are used in offline evaluation.
Motivated by this observation, we proposed a novel estimator for algorithm selection in A/B testing, referred to as MID.
MID induces positive correlation by constructing virtually shared data from the original A/B testing data through a hypothetical middle algorithm.
Extensive experiments on a public real-world dataset demonstrate that MID achieves a selection error rate comparable to conventional approaches while using only half to one quarter of the sample size.
Moreover, MID avoids the variance explosion caused by importance weighting even when comparing algorithms with low similarity, leading to more reliable algorithm selection than offline evaluation baselines.

Although this study instantiates MID using importance weighting, the estimator is flexible and can naturally incorporate more advanced techniques such as the Doubly Robust estimator.
By integrating such methods, we expect to further reduce algorithm selection error rates in A/B testing.

\bibliographystyle{ACM-Reference-Format}
\balance
\bibliography{6_citation}

\clearpage
\onecolumn
\appendix

\section{Proofs} \label{app:proof}
\subsection{Derivation of the Variance of MID}
Recall that MID estimates the performance difference between algorithms A and B in a stepwise manner via a middle algorithm M.
\begin{align}
    \Delta \hat{V}_{MID} = &\hat{V}_{IPS}(\pi_A;\mathcal{D}_A) - \hat{V}_{IPS}(\pi_M;\mathcal{D}_A) + \hat{V}_{IPS}(\pi_M;\mathcal{D}_B) - \hat{V}_{IPS}(\pi_B;\mathcal{D}_B) \notag
\end{align}
Then, the variance of MID is given by
\begin{align}
    Var[\Delta\hat{V}_{MID}]
    = &Var[\hat{V}_{IPS}(\pi_A;\mathcal{D}_A) - \hat{V}_{IPS}(\pi_M;\mathcal{D}_A) + \hat{V}_{IPS}(\pi_M;\mathcal{D}_B) - \hat{V}_{IPS}(\pi_B;\mathcal{D}_B)] \notag \\
    = &Var[\hat{V}_{IPS}(\pi_A;\mathcal{D}_A) - \hat{V}_{IPS}(\pi_M;\mathcal{D}_A)] + Var[\hat{V}_{IPS}(\pi_M;\mathcal{D}_B) - \hat{V}_{IPS}(\pi_B;\mathcal{D}_B)] \notag \\
    & -2 Cov[\hat{V}_{IPS}(\pi_A; \textcolor{red}{\mathcal{D}_A}) - \hat{V}_{IPS}(\pi_M; \textcolor{red}{\mathcal{D}_A}), \hat{V}_{IPS}(\pi_M;\textcolor{blue}{\mathcal{D}_B}) - \hat{V}_{IPS}(\pi_B;\textcolor{blue}{\mathcal{D}_B)}] \notag
\end{align}
Since the log data $\mathcal{D}_A$ and $\mathcal{D}_B$ are independent, the covariance term becomes zero.
Therefore, the variance of MID is expressed as follows, where we assume $n=n^A=n^B$, $\Delta\pi_A(a \mid x) = \pi_A(a \mid x) - \pi_M(a \mid x)$, $\Delta\pi_B(a \mid x) = \pi_B(a \mid x) - \pi_M(a \mid x)$.
\begin{align*}
    Var[\Delta\hat{V}_{MID}] = &\underset{\text{(1)}}{\underline{Var[\hat{V}_{IPS}(\pi_A;\mathcal{D}_A) - \hat{V}_{IPS}(\pi_M;\mathcal{D}_A)]}} + \underset{\text{(2)}}{\underline{Var[\hat{V}_{IPS}(\pi_M;\mathcal{D}_B) - \hat{V}_{IPS}(\pi_B;\mathcal{D}_B)]}} \\
    (1) = &\mathbb{V}_{p(\mathcal{D}_A)}\bigg[\frac{1}{n}\sum_{i=1}^{n}\frac{\pi_A(a_i^A \mid x_i^A) - \pi_M(a_i^A \mid x_i^A)}{\pi_A(a_i^A \mid x_i^A)}r_i^A\bigg] \\
    = &\frac{1}{n}\mathbb{V}_{p(x)\pi_A(a \mid x)p(r \mid x, a)}\bigg[\frac{\Delta\pi_A(a \mid x)}{\pi_A(a \mid x)}r\bigg] \\ 
    = &\frac{1}{n}\bigg\{\mathbb{E}_{p(x)\pi_A(a \mid x)}\bigg[\mathbb{V}_{p(r \mid x, a)}\bigg[\frac{\Delta\pi_A(a \mid x)}{\pi_A(a \mid x)}r\bigg]\bigg] + \mathbb{V}_{p(x)\pi_A(a \mid x)}\bigg[\mathbb{E}_{p(r \mid x, a)}\bigg[\frac{\Delta\pi_A(a \mid x)}{\pi_A(a \mid x)}r\bigg]\bigg]\bigg\} \\
    = & \frac{1}{n}\bigg\{\mathbb{E}_{p(x)\pi_A(a \mid x)}\bigg[\frac{\Delta\pi_A^2(a \mid x)}{\pi_A^2(a \mid x)}\sigma^2(x, a)\bigg] + \mathbb{V}_{p(x)\pi_A(a \mid x)}\bigg[\frac{\Delta\pi_A(a \mid x)}{\pi_A(a \mid x)}q(x, a)\bigg]\bigg\} \\
    = & \frac{1}{n}\bigg\{\underset{\text{(i)}}{\underline{\mathbb{E}_{p(x)}\bigg[\sum_{a \in \mathcal{A}}\frac{\Delta\pi_A^2(a \mid x)}{\pi_A(a \mid x)}\sigma^2(x, a)\bigg]}} + \mathbb{V}_{p(x)\pi_A(a \mid x)}\bigg[\frac{\Delta\pi_A(a \mid x)}{\pi_A(a \mid x)}q(x, a)\bigg]\bigg\} \\
    = &\frac{1}{n}\bigg\{\text{(i)} + \mathbb{E}_{p(x)}\bigg[\mathbb{V}_{\pi_A(a \mid x)}\bigg[\frac{\Delta\pi_A(a \mid x)}{\pi_A(a \mid x)}q(x, a)\bigg]\bigg] + \mathbb{V}_{p(x)}\bigg[\mathbb{E}_{\pi_A(a \mid x)}\bigg[\frac{\Delta\pi_A(a \mid x)}{\pi_A(a \mid x)}q(x, a)\bigg]\bigg]\bigg\} \\
    = &\frac{1}{n}\bigg\{\text{(i)} + \mathbb{E}_{p(x)}\bigg[\mathbb{E}_{\pi_A(a \mid x)}\bigg[\frac{\Delta\pi_A^2(a \mid x)}{\pi_A^2(a \mid x)}q^2(x, a)\bigg] - \mathbb{E}_{\pi_A(a \mid x)}^2\bigg[\frac{\Delta\pi_A(a \mid x)}{\pi_A(a \mid x)}q(x, a)\bigg]\bigg] + \mathbb{V}_{p(x)}\bigg[\mathbb{E}_{\pi_A(a \mid x)}\bigg[\frac{\Delta\pi_A(a \mid x)}{\pi_A(a \mid x)}q(x, a)\bigg]\bigg] \bigg\} \\
    = &\frac{1}{n}\bigg\{\text{(i)} + \underset{\text{(ii)}}{\underline{\mathbb{E}_{p(x)}\bigg[\sum_{a \in \mathcal{A}}\frac{\Delta\pi_A^2(a \mid x)}{\pi_A(a \mid x)}q^2(x, a)\bigg]\bigg]}} - \mathbb{E}_{p(x)}\bigg[\mathbb{E}_{\pi_A(a \mid x)}^2\bigg[\frac{\Delta\pi_A(a \mid x)}{\pi_A(a \mid x)}q(x, a)\bigg]\bigg] + \mathbb{V}_{p(x)}\bigg[\mathbb{E}_{\pi_A(a \mid x)}\bigg[\frac{\Delta\pi_A(a \mid x)}{\pi_A(a \mid x)}q(x, a)\bigg]\bigg] \bigg\} \\
    = &\frac{1}{n}\bigg\{\text{(i)} + \text{(ii)} \underset{=0}{\underline{- \mathbb{E}_{p(x)}\bigg[\mathbb{E}_{\pi_A(a \mid x)}^2\bigg[\frac{\Delta\pi_A(a \mid x)}{\pi_A(a \mid x)}q(x, a)\bigg]\bigg] + \mathbb{E}_{p(x)}\bigg[\mathbb{E}_{\pi_A(a \mid x)}^2\bigg[\frac{\Delta\pi_A(a \mid x)}{\pi_A(a \mid x)}q(x, a)\bigg]\bigg]}} - \mathbb{E}_{p(x)}^2\bigg[\mathbb{E}_{\pi_A(a \mid x)}\bigg[\frac{\Delta\pi_A(a \mid x)}{\pi_A(a \mid x)}q(x, a)\bigg]\bigg] \bigg\} \\
    = &\frac{1}{n}\bigg\{\text{(i)} + \text{(ii)} - \bigg(\mathbb{E}_{p(x)\pi_A(a \mid x)}\bigg[\frac{\Delta\pi_A(a \mid x)}{\pi_A(a \mid x)}q(x, a)\bigg]\bigg)^2 \bigg\} \\
    = &\frac{1}{n}\bigg\{\text{(i)} + \text{(ii)} - \bigg(V(\pi_A) - V(\pi_M)\bigg)^2 \bigg\} \\
    = &\frac{1}{n}\bigg\{\mathbb{E}_{p(x)}\bigg[\sum_{a \in \mathcal{A}}\bigg(\frac{\Delta\pi_A^2(a \mid x)}{\pi_A(a \mid x)}\bigg)\big(q^2(x, a) + \sigma^2(x, a)\big)\bigg] - \bigg(V(\pi_A) - V(\pi_M)\bigg)^2 \bigg\} \\
    (2) = &\frac{1}{n}\bigg\{\mathbb{E}_{p(x)}\bigg[\sum_{a \in \mathcal{A}}\bigg(\frac{\Delta\pi_B^2(a \mid x)}{\pi_B(a \mid x)}\bigg)\big(q^2(x, a) + \sigma^2(x, a)\big)\bigg] - \bigg(V(\pi_B) - V(\pi_M)\bigg)^2 \bigg\} \\
    \therefore Var&[\Delta\hat{V}_{MID}] = \frac{1}{n}\bigg\{\mathbb{E}_{p(x)}\bigg[\sum_{a \in \mathcal{A}}\bigg(\frac{\Delta\pi_A^2(a \mid x)}{\pi_A(a \mid x)} + \frac{\Delta\pi_B^2(a \mid x)}{\pi_B(a \mid x)}\bigg)\big(q^2(x, a) + \sigma^2(x, a)\big)\bigg] - \bigg(V(\pi_A) - V(\pi_M)\bigg)^2 - \bigg(V(\pi_B) - V(\pi_M)\bigg)^2\bigg\}
\end{align*}

\subsection{Proof of Theorem \ref{thm:mid-policy}}
\begin{proof}
The variance of MID consists of the following reward term and performance term.
\begin{align*}
    &R(x, a; \Delta\hat{V}_{MID}) = \bigg(\frac{(\pi_A(a \mid x)-\pi_M(a \mid x))^2}{\pi_A(a \mid x)} + \frac{(\pi_B(a \mid x)-\pi_M(a \mid x))^2}{\pi_B(a \mid x)}\bigg)\big(q^2(x, a) + \sigma^2(x, a)\big)\\
    &\text{Per}(\Delta\hat{V}_{MID}) = \big(V(\pi_A) - V(\pi_M)\big)^2 + \big(V(\pi_B) - V(\pi_M)\big)^2
\end{align*}
As discussed in the main text, since the performance term is unknown, we derive the optimal middle algorithm $\pi_{M^*}$ by minimizing the reward term to reduce the selection error rate.
Specifically, we derive this algorithm by differentiating the reward term and setting it to zero.
\begin{gather}
    \frac{d}{d\pi_M}R(x, a; \Delta\hat{V}_{MID}) = -\frac{2(\pi_A(a \mid x) - \pi_M(a \mid x))}{\pi_A(a \mid x)} -\frac{2(\pi_B(a \mid x) - \pi_M(a \mid x))}{\pi_B(a \mid x)} \notag\\
    \therefore \pi_{M^*}(a \mid x) = \frac{2\pi_A(a \mid x)\pi_B(a \mid x)}{\pi_A(a \mid x) + \pi_B(a \mid x)} \notag
\end{gather}
\end{proof}

\subsection{Proof of Theorem \ref{thm:cond-unbiased}}
\begin{proof}
We prove that MID is unbiased, provided that a middle algorithm $\pi_M(a \mid x)$ satisfies the following conditions.
\begin{align}
    \pi_A(a \mid x) = 0 \oplus \pi_B(a \mid x) = 0 \Rightarrow \pi_M(a \mid x) = 0 \notag
\end{align}
Let $\mathcal{A}_A, \mathcal{A}_B$ be the supports of algorithms $\pi_A(a \mid x), \pi_B(a \mid x)$ (i.e., the action sets with positive probabilities).
We consider the conditions on the action sets of the middle algorithm $\mathcal{A}_M$ required for MID to be unbiased.
Then, the expected value of MID is given as follows.
\begin{align*}
    \mathbb{E}_{p(\mathcal{D}_A, \mathcal{D}_B)}[\Delta\hat{V}_{MID}]
    &= \underset{\text(1)}{\underline{\mathbb{E}_{p(\mathcal{D}_A)}[\hat{V}_{IPS}(\pi_A; \mathcal{D}_A) - \hat{V}_{IPS}(\pi_M; \mathcal{D}_A)]}} + \underset{\text{(2)}}{\underline{\mathbb{E}_{p(\mathcal{D}_B)}[\hat{V}_{IPS}(\pi_M; \mathcal{D}_B) - \hat{V}_{IPS}(\pi_B; \mathcal{D}_B)]}}, \\
    (1) &= \mathbb{E}_{p(x)\pi_A(a \mid x)p(r \mid x, a)}\bigg[\frac{\pi_A(a \mid x)}{\pi_A(a \mid x)}r\bigg] - \mathbb{E}_{p(x)\pi_A(a \mid x)p(r \mid x, a)}\bigg[\frac{\pi_M(a \mid x)}{\pi_A(a \mid x)}r\bigg] \\
    &= V(\pi_A) - \mathbb{E}_{p(x)}\bigg[\sum_{a \in \mathcal{A}_A}\pi_M(a \mid x)q(x, a)\bigg]\\
    &= V(\pi_A) - V(\pi_M) + \mathbb{E}_{p(x)}\bigg[\sum_{a \in \overline{\mathcal{A}_A} \cap \mathcal{A}_M}\pi_M(a \mid x)q(x, a) \bigg]\\
    &= V(\pi_A) - V(\pi_M) + \mathbb{E}_{p(x)}\bigg[\sum_{a \in \overline{\mathcal{A}_A} \cap \mathcal{A}_B \cap \mathcal{A}_M}\pi_M(a \mid x)q(x, a) \bigg] + \mathbb{E}_{p(x)}\bigg[\sum_{a \in \overline{\mathcal{A}_A} \cap \overline{\mathcal{A}_B} \cap \mathcal{A}_M}\pi_M(a \mid x)q(x, a) \bigg], \\
    (2) &= V(\pi_M) - V(\pi_B) - \mathbb{E}_{p(x)}\bigg[\sum_{a \in \overline{\mathcal{A}_B} \cap \mathcal{A}_M}\pi_M(a \mid x)q(x, a) \bigg]\\
    &= V(\pi_M) - V(\pi_B) - \mathbb{E}_{p(x)}\bigg[\sum_{a \in \mathcal{A}_A \cap \overline{\mathcal{A}_B} \cap \mathcal{A}_M}\pi_M(a \mid x)q(x, a) \bigg] - \mathbb{E}_{p(x)}\bigg[\sum_{a \in \overline{\mathcal{A}_A} \cap \overline{\mathcal{A}_B} \cap \mathcal{A}_M}\pi_M(a \mid x)q(x, a) \bigg], \\
    \therefore \mathbb{E}_{p(\mathcal{D}_A,\mathcal{D}_B)}[\Delta\hat{V}_{MID}] &= V(\pi_A) - V(\pi_B) + \mathbb{E}_{p(x)}\bigg[\sum_{a \in \overline{\mathcal{A}_A} \cap \mathcal{A}_B \cap \mathcal{A}_M}\pi_M(a \mid x)q(x, a) \bigg] - \mathbb{E}_{p(x)}\bigg[\sum_{a \in \mathcal{A}_A \cap \overline{\mathcal{A}_B} \cap \mathcal{A}_M}\pi_M(a \mid x)q(x, a) \bigg]
\end{align*}
Thus, to eliminate the bias of MID, we require the middle algorithm's action sets $\mathcal{A}_M$ such that $\mathcal{A}_A \cap \overline{\mathcal{A}_B} \cap \mathcal{A}_M = \emptyset$ and $\overline{\mathcal{A}_A} \cap \mathcal{A}_B \cap \mathcal{A}_M = \emptyset$.
In other words, MID becomes unbiased if $\pi_M(a \mid x)$ is zero where either $\pi_A(a \mid x)$ or $\pi_B(a \mid x)$ is zero.
\end{proof}

\section{Additional Experimental Results} \label{app:additional}
In this section, we show the results of algorithm selection experiments under various conditions that are omitted from the main text.

Figure \ref{fig:result-reward-noise} illustrates the robustness of each estimator against reward noise by comparing how the error rate changes as the magnitude of noise added to the observed rewards varies.
In real-world video recommendation tasks, observed rewards (e.g., watch ratio) contain noise due to various factors.
For instance, a user's viewing behavior depends on their mood, the time of day, and the viewing environment.
Consequently, an environment with high reward noise corresponds to a scenario in which the reliability of these observed rewards is low.
In this experiment, we verify the capability of each estimator to make accurate decisions under such noisy conditions.
The experimental results demonstrate that the proposed method exhibits superior robustness against reward noise.
First, higher reward noise leads to increased variance and the higher error rate for all estimators.
However, focusing on the rate of variance increase relative to the noise, the proposed method exhibits a moderate increase.
This allows the proposed method to maintain the lower error rate than AVG and IPS in all conditions, ensuring its robustness even when noise levels are extremely high.
The robustness against reward noise is also attributed to the variance reduction effect of the proposed method demonstrated in this paper.
It indicates that even in real-world scenarios with unreliable data, the proposed method provides a more stable decision-making capability compared to conventional A/B testing.

Figure \ref{fig:result-action-num} compares how the error rate of each estimator changes as we vary the number of actions for the recommendation algorithm.
In the context of real-world video recommendation, the number of actions corresponds to the size of the action space, namely the total number of videos available for recommendation.
The purpose of this experiment is to verify a known weakness of IPS.
Typically, as the number of actions increases, the selection probability for each action relatively decreases, making the estimator susceptible to explosive variance growth due to importance weights.
This leads to unstable estimates and makes correct algorithm selection more difficult.
In this experiment, we verify the robustness of the proposed method against increases in the number of actions.
The experimental results demonstrate that the proposed method effectively suppresses the variance increase derived from the large number of actions.
With IPS, variance increases rapidly when the number of actions exceeds approximately 25, leading to a significant increase in the error rate.
In contrast, MID reduces the variance associated with a larger number of actions, despite also using importance weights.
This is also attributed to the introduction of the middle algorithm $\pi_M$, which avoids extreme importance weights and stabilizes the estimation.
Additionally, MID keeps variance lower than that of AVG, resulting in the consistently lower error rate for any number of actions.
Thus, our method allows for stable and accurate decision-making even in large action spaces where IPS fails.
Furthermore, the results imply that our estimator is superior to AVG across various numbers of actions.

\begin{figure*}
\centering
\includegraphics[width=\linewidth]{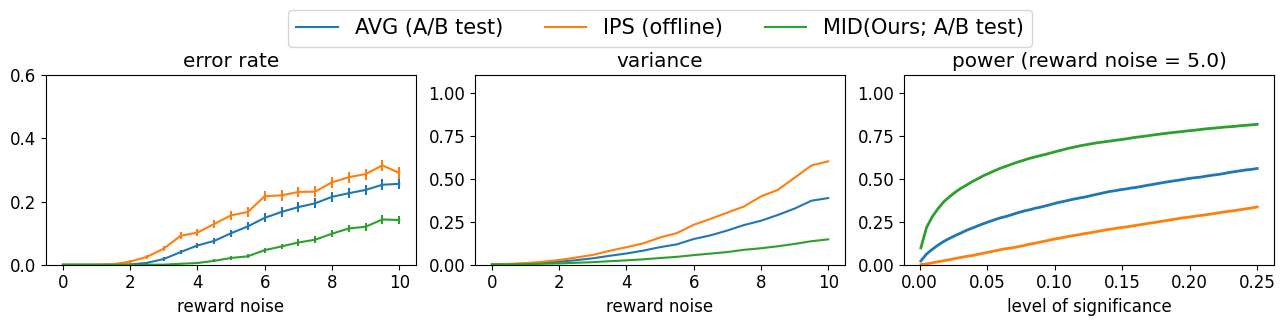}
\vspace{-3mm}
\caption{Comparison of the estimator's error rate, variance, and statistical power with varying reward noises}
\vspace{-3mm}
\label{fig:result-reward-noise}
\end{figure*}

\begin{figure*}
\centering
\includegraphics[width=\linewidth]{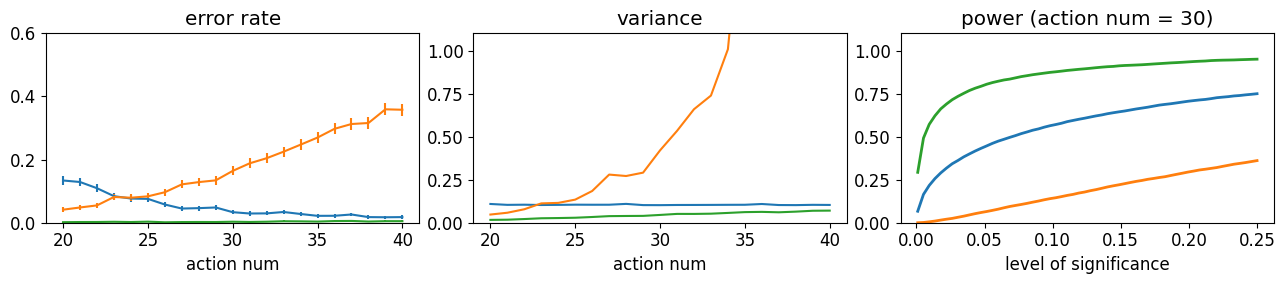}
\vspace{-3mm}
\caption{Comparison of the estimator's error rate, variance, and statistical power with varying number of actions}
\vspace{-3mm}
\label{fig:result-action-num}
\end{figure*}

\end{document}